\documentclass{article}
\usepackage[preprint]{neurips_2019}

\usepackage{times}
\usepackage{soul}
\usepackage{url}
\usepackage[hidelinks]{hyperref}
\usepackage[utf8]{inputenc}
\usepackage[small]{caption}
\usepackage{graphicx}
\usepackage{amsmath}
\usepackage{booktabs}
\usepackage{algorithm}
\usepackage{algorithmic}
\usepackage{dsfont}
\usepackage[none]{hyphenat}
\usepackage{amsfonts}
\usepackage{caption}
\usepackage{amsthm}
\usepackage{subcaption}
\usepackage{float}
\usepackage{thmtools} 
\usepackage{thm-restate}

\newtheorem*{corollary*}{Corollary}
\newtheorem{lemma}{Lemma}
\newtheorem*{lemma*}{Lemma}

\newtheorem{defn}{Definition}
\newcommand{\Lip}[3]{\text{Lip}_{#1}^{#2}{#3}}
\newcommand{\norm}[1]{\left\lVert#1\right\rVert}
\newcommand{\fancyS}[0]{\mathcal{S}}
\newcommand{\fancyX}[0]{\mathcal{X}}
\newcommand{\fancyY}[0]{\mathcal{Y}}
\newcommand{\fancyA}[0]{\mathcal{A}}
\newcommand{\fancyF}[0]{\mathcal{F}}
\newcommand{\fancyT}[0]{\mathcal{T}}
\newcommand{\fancyG}[0]{\mathcal{G}}
\newcommand{\MCube}[0]{$\mathrm{M^3}$}

\newcommand{\integer}[0]{\mathbb{Z}}
\newcommand{\ReLU}[0]{\text{ReLU}}
\newcommand{\rademacher}[1]{\textrm{Rad}(#1)}
\newcommand{\aBold}{\textbf{a}}
\newcommand{\identityFunction}{ \mathds{1}}
\newcommand{\That}[0]{\widehat T}
\newcommand{\aVec}[0]{\textbf{a}}
\newcommand{\expected}[2]{\mathbb{E}_{#1}\Big[ #2 \Big]}
\usepackage{xcolor}

\title{Combating the Compounding-Error Problem \\ with a Multi-step Model}
\author{%
  Kavosh Asadi \\
  Department of Computer Science\\
  Brown University\\
  \And
  Dipendra Misra \\
  Department of Computer Science \\
  Cornell University \\
   \AND
  Seungchan Kim \\
  Department of Computer Science \\
  Brown University \\
  \And
  Michael L. Littman \\
  Department of Computer Science \\
  Brown University \\
}

\begin{document}

\maketitle
\begin{abstract}
    Model-based reinforcement learning is an appealing framework for creating agents that learn, plan, and act in sequential environments. Model-based algorithms typically involve learning a transition model that takes a state and an action and outputs the next state---a one-step model. This model can be composed with itself to enable predicting multiple steps into the future, but one-step prediction errors can get magnified, leading to unacceptable inaccuracy. This compounding-error problem plagues planning and undermines model-based reinforcement learning. In this paper, we address the compounding-error problem by introducing a multi-step model that directly outputs the outcome of executing a sequence of actions. Novel theoretical and empirical results indicate that the multi-step model is more conducive to efficient value-function estimation, and it yields better action selection compared to the one-step model. These results make a strong case for using multi-step models in the context of model-based reinforcement learning.
\end{abstract}
\section{Introduction}
The model-based approach to reinforcement learning (RL) offers a unique framework for addressing three important artificial intelligence (AI) problems: understanding the dynamics of the environment through interaction, using the acquired knowledge for planning, and performing sequential decision making. One promise of model-based RL is to enable sample-efficient learning~\citep{sutton1990integrated,deisenroth_pilco,levine2014learning}. This advantage is well-understood in theory in settings such as value-function estimation \citep{azar2012sample}, exploration in discounted and finite-horizon Markov decision processes (MDPs) \citep{szita2010model,dann2015sample}, and exploration in contextual decision processes \citep{sun2018model}.

Planning has long been a fundamental research area in AI \citep{hart1972correction, russel_norvig}. In the context of model-based RL, \cite{sutton_the_book} define planning as any process that goes from a model to a value function or a policy. They describe two notions of planning, namely decision-time planning and background planning. In decision-time planning, the agent utilizes the model during action selection by performing tree search \citep{kearns_near_optimal,kocsis_ucb,silver_go}. An agent is said to be performing background planning if it utilizes the model to update its value or policy network. Examples of background planning include the Dyna Architecture~\citep{sutton1990integrated,sutton_linear_dyna}, or model-based policy-search~\citep{abbeel_inaccurate,deisenroth_pilco,ensemble_kurutach}. We call an algorithm model-based if it performs either background or decision-time planning.

A key aspect distinguishing model-based RL from traditional planning is that the model is learned from experience. As such, the model may be imperfect due to ineffective generalization~\citep{abbeel_inaccurate,nagabandi2018neural}, inadequate exploration~\citep{r_max,model_based_active_exploration}, overfitting~\citep{asadi_lipschitz}, or irreducible errors in unrealizable settings~\citep{shai_shai,talvitie_hallucination_14}. More generally, a common view across various scientific disciplines is that all models are wrong, though some are still useful~\citep{box_all_models_are_wrong,wit_all_models}. 

Previous work explored learning models that look a single step ahead. To predict $H$ steps ahead, the starting point of a step $h \in [2,H]$ is set to the end point of the previous step $h-1$. Unfortunately, when the model is wrong, this procedure can interfere with successful planning~\citep{talvitie_hallucination_14,venkatraman_multi_step,asadi_lipschitz}. A significant reason for this failure is that the model may produce a ``fake'' input, meaning an input that cannot possibly occur in the domain, which is then fed back to the unprepared one-step model. Notice that two sources of error co-exist: The model is imperfect, and the model gets an inaccurate input in all but the first step. The interplay between the two errors leads to what is referred to as the \emph{compounding-error problem}~\citep{asadi_lipschitz}. To mitigate the problem, \cite{talvitie_hallucination_14} and \cite{venkatraman_multi_step} provided an approach, called hallucination, that prepares the model for the fake inputs generated by itself. In contrast, our approach to the problem is to avoid feeding such fake inputs altogether by using a multi-step model. Though multi-step algorithms are popular in the model-free setting~\citep{sutton_the_book,de_asis_multi,singh1996reinforcement,precup2000eligibility}, extension to the model-based setting remains less explored.

Our main contribution is to propose a Multi-step Model for Model-based RL (or simply \MCube) that directly predicts the outcome of executing a \emph{sequence} of actions. Learning the multi-step model allows us to avoid feeding fake inputs to the model. We further introduce a novel rollout procedure in which the original first state of the rollout will be the starting point across all rollout steps. Our theory shows that, relative to the one-step model, learning the multi-step model is more effective for value estimation. To this end, we study the hardness of learning the multi-step model through the lens of Rademacher complexity \citep{bartlett_rademacher}. Finally, we empirically evaluate the multi-step model and show its advantage relative to the one-step model in the context of background planning and decision-time planning.
\section{Background and Notation}

To formulate the reinforcement-learning problem, we use finite-horizon Markov decsion processes (MDPs) with continuous states and discrete actions. See \cite{puterman_mdps} for a thorough treatment of MDPs, and \cite{sutton_the_book} for an introduction to reinforcement learning.
\subsection{Lipschitz Continuity}

Following previous work \citep{berkenkamp2017safe,asadi_lipschitz,luo2018algorithmic,deep_mdp} we make assumptions on the smoothness of models, characterized below.
\begin{defn}
\label{defn:Lipschitz}
Given metric spaces $(\fancyX,d_1)$ and $(\fancyY,d_2)$, $f\!:\!\fancyX\!\to\!\fancyY$ is \emph{Lipschitz} if ,
\begin{equation*}
	Lip(f):=\sup_{x_1\in \fancyX,x_2\in \fancyX}\frac{d_2\big(f(x_1),f(x_2)\big)}{d_1(x_1,x_2)}
\end{equation*} 
is finite. Similarly, $f:\fancyX \times \fancyA \to \fancyY$ is \emph{uniformly Lipschitz} in $\fancyA$ if the quantity below is finite: \begin{equation*}
	Lip^\fancyA(f):=\sup_{a\in \fancyA}\sup_{x_1\in \fancyS,x_2\in \fancyS}\frac{d_2\big(f(x_1,a),f(x_2,a)\big)}{d_1(x_1,x_2)}\ .
\end{equation*}
\end{defn}
\renewcommand{\arraystretch}{1.25}
\begin{table}[H]
\begin{center}
\begin{tabular}{ |c|c|c| } 
 \hline
  & Definition & Agent's Approximation\\
 \hline
 $\pi(a\mid s)$& probability of taking $a$ in $s$& N/A\\
 $\aVec_h$ & $\langle a_1,...,a_h \rangle$ & N/A\\
 $T_{h}(s,\aVec_h)$& MDP state after taking  $\aVec_h$ in $s$ &$\That_{h}(s,\aVec_h)$\\
 $T_{h}(s,s',\pi)$&$\textrm{Pr}(s_{t+h}=s'\mid s_{t}=s,\pi)$&$\widehat T_{h}(s,s',\pi)$\\
 $(T_{1})^h(s,\aVec_h)$ & $T_1\Big(... T_1\big(T_1(s,a_1)\big),a_2,...,a_h\Big)$&$(\That_{1})^h(s,\aVec_h)$ \\
 \hline
\end{tabular}
 \caption{Notation used in the paper.}
  \label{table:transition_notation}
 \end{center}
\end{table}
\subsection{Rademacher Complexity}
\newcommand{\namecite}[1]{\citeauthor{#1}~[\citeyear{#1}]}

We use Rademacher complexity for sample complexity analysis. We define this measure, but for details see \namecite{bartlett_rademacher} or \namecite{mohri_foundation}. Also, see \cite{nan_planning} and \cite{lehnert2018value} for previous applications of Rademacher in reinforcement learning.
\begin{defn}
Consider $f:\fancyS~\to~[-1,1]$, and a set of such functions $\fancyF$. The Rademacher complexity of this set, $\rademacher{\fancyF}$, is defined as:
$$\rademacher{\fancyF}:=\expected{s_j,\sigma_j}{\sup_{f\in\fancyF}\frac{1}{n}\sum_{j=1}^n \sigma_j f(s_j)}\ ,$$
where $\sigma_j$, referred to as Rademacher random variables, are drawn uniformly at random from $\{\pm 1\}$.\label{defn:rademacher}
\end{defn}
The Rademacher variables could be thought of as independent and identically distributed noise. Under this view, the average $\frac{1}{n}\sum_{j=1}^n \sigma_j f(s_j)$ quantifies the extent to which $f(\cdot)$ matches the noise. We have a high Rademacher complexity for a complicated hypothesis space that can accurately match noise. Conversely, a simple hypothesis space has a low Rademacher complexity.

To apply Rademacher to the model-based setting where the output of the model is a vector, we extend Definition \ref{defn:rademacher} to vector-valued functions that map to $[-1,1]^{d}$. Consider a function $g:=\langle f_1,...,f_d \rangle$ where $\forall i\ f_i \in \fancyF$. Define the set of such functions $\fancyG$. Then:
$$\rademacher{\fancyG}:=\expected{s_j,\sigma_{ji}}{\sup_{g\in\fancyG}\frac{1}{n}\sum_{j=1}^n\sum_{i=1}^d \sigma_{ji} g(s_j)_i}\ ,$$
where $\sigma_{ji}$ are drawn uniformly from $\{\pm 1\}$.

\subsection{Transition-Model Notation}

Our theoretical results focus on MDPs with deterministic transitions. This common assumption~\citep{abbeel_inaccurate,talvitie_2017} simplifies notation and analysis. Note that a deterministic environment, like nearly all Atari games~\citep{bellemare_atari}, can still be quite complex and challenging for model-based reinforcement learning~\citep{kamyar_surprising}. We introduce an extension of our results to a stochastic setting in the Appendix. 

We use an overloaded notation for a transition model (see Table~\ref{table:transition_notation}). Notice that our definitions support stochastic $h$-step transitions since these transitions are policy dependent, and that we allow for stochastic policies.

\section[junk]{$\mathrm{M^3}$ -- A Multi-step Model for Model-based Reinforcement Learning}

$\mathrm{M^3}$ is an extension of the one-step model---rather than only predicting a single step ahead, it learns to predict $h\in\{1,...,H\}$ steps ahead using $H$ different functions:
$$\That_h(s,\aVec_h) \approx T_h(s,\aVec_h)\ ,$$
where $\aBold_h=\langle a_1,a_2,...,a_h\rangle$. Finally, by $\mathrm{M^3}$, we mean the set of these $H$ functions: $$\mathrm{M^3}:=\big\{\That_h \mid h \in \integer:h\in [1,H]\big\}\ .$$
This model is different than the few examples of multi-step models studied in prior work: 
\cite{sutton_mixture} as well as
\cite{van_seijen_deeper}
considered multi-step models that do not take actions as input, but are implicitly conditioned on the current policy. Similarly, option models are multi-step models that are conditioned on one specific policy and a termination condition \citep{precup_multi,sutton_option,silver2012compositional}. Finally, \cite{silver2017predictron} introduced a multi-step model that directly predicts next values, but the model is defined for prediction tasks. 

We now introduce a new rollout procedure using the multi-step model. Note that by an $H$-step rollout we mean sampling the next action using the agent's fixed policy $\pi:\fancyS\mapsto \Pr(\fancyA)$, then computing the next state using the agent's model, and then iterating this procedure for $H-1$ more times. We now show a novel rollout procedure using the multi-step model that obviates the need for the model to get its own output.
\begin{figure}
\begin{subfigure}[]{\textwidth}
\centering
   \includegraphics[width=180pt]{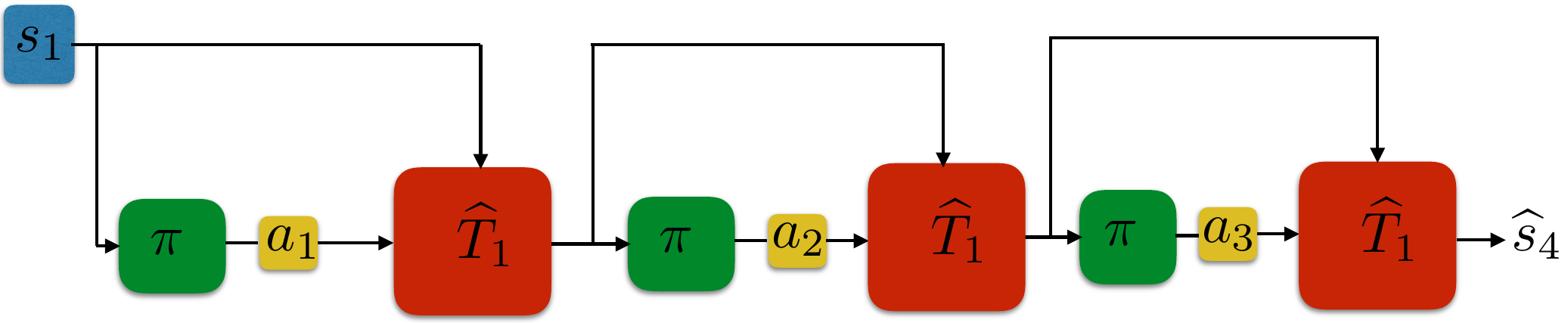}
\end{subfigure}

\begin{subfigure}[]{\textwidth}
\centering
   \includegraphics[width=180pt]{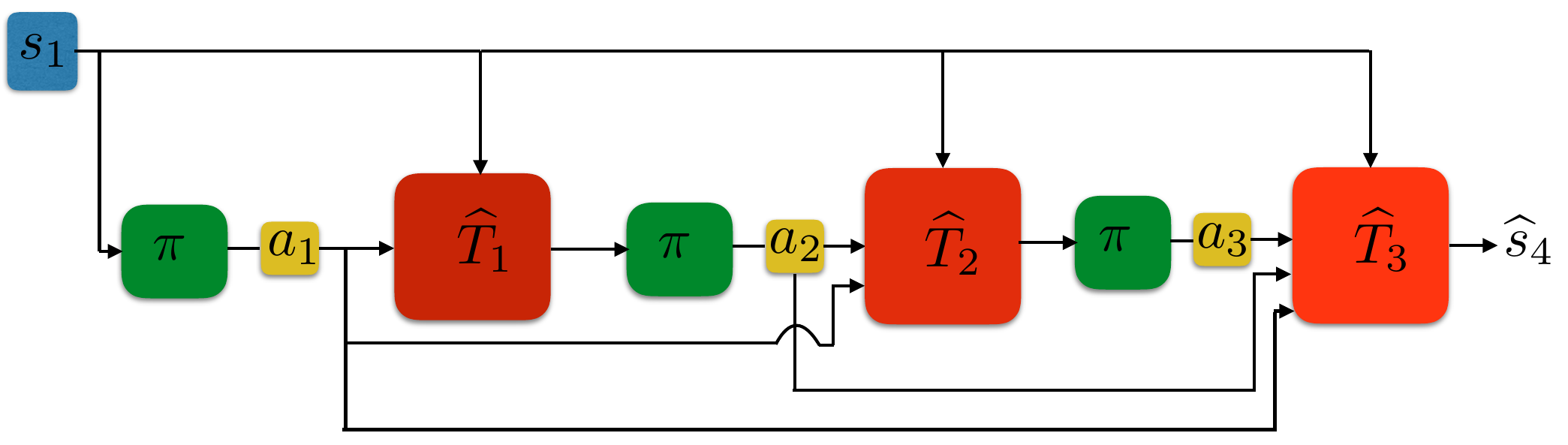}
\end{subfigure}
\caption[Rollouts]{(top) a 3-step rollout using a one-step model. (bottom) a 3-step rollout using a multi-step model \MCube. Crucially, at each step of the multi-step rollout, the agent uses $s_1$ as the starting point. The output of each intermediate step is only used to compute the next action.}
\label{fig:rollout}
\end{figure}
To this end, we derive an approximate experession for:
$$T_{H}(s,s',\pi):=\textrm{Pr}(s_{t+H}=s'\mid s_{t}=s,\pi)\ .$$
Key to our approach is to rewrite $\That_{H}(s,s',\pi)\approx T_{H}(s,s',\pi)$ in terms of predictions conditioned on action sequences as shown below:
\begin{eqnarray*}
\!\That_{H}(s,s',\pi)&:=&\textrm{Pr}(s_{t+H}=s'\mid s_{t}=s,\pi)=\!\sum_{\aVec_H}\textrm{Pr}(\aVec_H \mid s,\pi)\ \identityFunction\big(s'=\underbrace{\That_{H}(s,\aVec_H)}_{\textrm{available by \MCube}}\big) \ .
\end{eqnarray*}
Observe that given the \MCube\ model introduced above, $\That_H(s,\aVec_h)$ is actually available---we only need to focus on the quantity $\textrm{Pr}(\aVec_H|s,\pi)$. Intuitively, we need to compute the probability of taking a sequence of actions of length $H$ in the next $H$ steps starting from $s$. This probability is clearly determined by the states observed in the next $H-1$ steps, and could be written as follows:
\begin{eqnarray*}
\textrm{Pr}(\aVec_H|s,\pi)\!=\!\textrm{Pr}(a_{H}|\aVec_{H-1},s,\pi)\textrm{Pr}(\aVec_{H-1}|s_{t}=s,\pi)\!&=&\!\textrm{Pr}\big(a_{H}|\That_{H-1}(s,\aVec_{H-1}),\pi\big)\textrm{Pr}(\aVec_{H-1}|s,\pi)\\
&=&\pi\big(a_{H}|\underbrace{\That_{H-1}(s,\aVec_{H-1})}_{\textrm{available by \MCube}}\big)\textrm{Pr}(\aVec_{H-1}|s,\pi)\ .\\
\end{eqnarray*}
We can compute $\textrm{Pr}(\aVec_H|s,\pi)$ if we have $\textrm{Pr}(\aVec_{H-1}|s,\pi)$. Continuing for $H-1$ steps:
$$\textrm{Pr}(\aVec_{H}\mid s,\pi)=\pi(a_{1}\mid s)\prod_{h=2}^{H}\pi\big(a_{h}\mid \underbrace{\That_{h-1}(s,\aVec_{h-1}}_{\textrm{available by \MCube}})\big)\ ,$$
which we can compute given the $H-1$ first functions of \MCube, namely $\That_h$ for $h\in [1,H-1]$. 

Finally, to compute a rollout, we  sample from $\That_{H}(s,s',\pi)$ by sampling from the policy at each step:
$$\widehat s_{H+1} = \That_{H}(s,\aVec_H)\quad \textrm{where} \quad a_{h}\sim \pi\big(\cdot\mid\That_{h-1}(s,\aVec_{h-1})\big)\ . $$
Notice that, in the above rollout with \MCube, we have used the first state $s$ as the starting point of every single rollout step. Crucially, we do not feed the intermediate state predictions to the model as input. We hypothesize that this approach can combat the compounding error problem by removing one source of error, namely feeding the model a noisy input, which is otherwise present in the rollout using the one-step model. We illustrate the rollout procedure in Figure~\ref{fig:rollout} for a better juxtaposition of this new rollout procedure and the standard rollout procedure performed using the one-step model. In the rest of the paper, we test this hypothesis in theory and practice.
\section{Value-Function Error Bounds}

We argued above that a multi-step model can better deal with the compounding-error problem. We now formalize this claim in the context of policy evaluation. Specifically, we show a bound on value-function estimation error of a fixed policy in terms of the error in the agent's model, while highlighting similar bounds from prior work~\citep{ross_dagger,talvitie_2017,asadi_lipschitz}. All proofs can be found in the Appendix. Note also that in all expectations below actions and states are distributed according to the agent's fixed policy and its stationary distribution, respectively.
\begin{restatable}{theorem}{oneStepValueError}
	\label{theorem:value_error_with_one_step_model}
	Define the $H$-step value function $V^{\pi}_{H}(s):=\expected{s_i,a_i}{\sum_{i=1}^{H} R(s_i,a_i)}$, then
	$$\Big|\expected{s_1}{V^{\pi}_{H}(s_1)-\widehat V^{\pi}_{H}(s_1)}\Big|
			\leq  \Lip{}{\fancyA}{(R)}\sum_{h=1}^{H-1}(H-h)\expected{s_h,a_h}{ \norm{T_1(s_h,a_h) -\widehat T_1(s_h,a_h)}}\ .
		\label{eq:value_error_with_one_step_model}$$
\end{restatable}
Moving to the multi-step case with \MCube$=\big\{\That_h:h\in [1,H]\big\}$, we have the following result:
\begin{restatable}{theorem}{multiStepValueError}
\label{theorem:value_error_with_multi_step_model}
	\begin{equation*}
				\Big|\expected{s_1}{V^{\pi}_{H}(s_1)-\widehat V^{\pi}_{H}(s_1)}\Big|\leq  \Lip{}{\fancyA}{(R)}\sum_{h=1}^{H-1}\expected{s_1,\aVec_{h}}{ \norm{T_h(s_1,\aVec_h) -\That_h(s_1,\aVec_{h})}}\ .
	\label{eq:value_error_with_multi_step_model}
	\end{equation*}
\end{restatable}
Note that, by leveraging \MCube\ in Theorem~\ref{theorem:value_error_with_multi_step_model}, we removed the $H-h$ factor from each summand of the bound in Theorem~\ref{theorem:value_error_with_one_step_model}. This result suggests that \MCube\ is more conducive to value-function learning as long as $h$-step generalization error $\expected{s_1,\aVec_{h}}{ \norm{T_h(s_1,\aVec_h) -\That_h(s_1,\aVec_{h})}}$ grows slowly with $h$. In the next section, we show that this property holds under weak assumptions, concluding that the bound from Theorem~\ref{theorem:value_error_with_multi_step_model} is an improvement over the bound from Theorem~\ref{theorem:value_error_with_one_step_model}. 
\section{Analysis of Generalization Error}

We now study the sample complexity of learning $h$-step dynamics. To formulate the problem, consider a dataset $D:\langle s^i,\aVec^i_h:=\langle a^{i}_1,...,a^{i}_h\rangle,T_{h}(s^i,\aVec^i_h)\rangle$. Consider a scalar-valued function $f:[-1,1]^{d}\to[-1,1]$, where $d$ is the dimension of the state space, and a set of such functions $\fancyF$. Consider a function $\That_h=\langle f_1,...,f_d \rangle$ and a set of functions 
$$\fancyT_h:=\big\{\That_h:=\langle f_1,...,f_d \rangle |f_i\in\fancyF\ \forall i \in \{1,...,d\}\big\}.$$
For each $h$ we learn a function $\That_h$ such that: $\min_{\That_h} \frac{1}{|D|}\sum_{i=1}^{|D|}\norm{T_h(s^i,\aVec^i_h)-\That_h(s^i,\aVec^i_h)}_{1}\!.$
\begin{restatable}{lemma}{RademacherLemma}
\label{lemma:generalization_error_bound_using_rademacher}
For any function $\That\in \fancyT_h$, $\forall\delta \in (0,1)$, training error $\Delta$, and probability at least $1-\delta$:
\begin{equation*}
	\underbrace{\expected{s,\aVec_h}{\!\norm{T_h(s,\aVec_h)-\That_h(s,\aVec_h)}_{1}}}_{\textup{generalization error}}\leq2\sqrt{2}\rademacher{\fancyT_h} +d\ln{\sqrt{\frac{1/\delta}{|D|}}}+\Delta\ ,
\end{equation*}
\end{restatable}
  The second and third terms of the bound are not functions of $h$. Therefore, to understand how generalization error grows as a function of $h$, we look at the dependence of Rademacher complexity of $\fancyT_h$ to $h$.
\begin{restatable}{lemma}{rademacherNeuralNets}
\label{lemma:rademacher_neural_nets}
	For the hypothesis space  shown in \emph{Figure 7} 
	(see Appendix):
	$$\rademacher{\fancyT_h}\leq\frac{d}{|D|}\Big(\expected{s^i}{  \sqrt{\sum_{i=1}^{|D|}\norm{ s^i}_{2}} }+\expected{\aVec^i_h}{  \sqrt{\sum_{i=1}^{|D|}\norm{ \aVec^i_h}_{2}} }\Big).$$
\end{restatable}
In our experiments, we represent action sequences with $h$-hot vector so $\expected{\aVec^i_h}{  \sqrt{\sum_{i=1}^{n}\norm{ \aVec^i_h}_2} }=h$.
We finally get the desired result:
\begin{restatable}{theorem}{finalTheorem}
Define the constant $C_1:=\Lip{}{\fancyA}{(R)}\Big(\frac{\!2\sqrt{2}d}{|D|}\expected{s^i}{  \!\sqrt{\!\sum_{i=1}^{|D|}\norm{ s^i}_{2}} }+d\ln{\sqrt{\frac{1/\delta}{|D|}}}+\Delta\Big)\ $ and the constant $C_2:=\Lip{}{\fancyA}{(R)}\frac{2\sqrt{2}d}{|D|}$. Then:
\begin{itemize}
    \item with the one-step model,$\ \Big|\expected{s_1}{V^{\pi}_{H}(s_1)-\widehat V^{\pi}_{H}(s_1)}\Big|\leq \frac{H(H-1)}{2}C_1+ \frac{H(H-1)}{2}C_2 .$
    \item with the multi-step model,
    $\ \Big|\expected{s_1}{V^{\pi}_{H}(s_1)-\widehat V^{\pi}_{H}(s_1)}\Big|\leq (H-1)C_1+\frac{H(H-1)}{2}C_2 .$
\end{itemize}
\end{restatable}
Note the reduction of factor $H$ in the coefficient of $C_1$, which is typically larger than $C_2$.
\section{Discussion}
As training the multi-step involves training $H$ different functions, the computational complexity of learning \MCube\ is $H$-times more than the complexity of learning the one-step model. However, the two rollout procedures shown in Figure \ref{fig:rollout} require equal computation. So in cases where planning is the bottleneck, the overall computational complexity of the two cases should be similar. Specifically, in experiments we observed that the overall running time under the multi-step model was always less than double the running time under the one-step model.

Previous work has identified various ways of improving one-step models: quantifying model uncertainty during planning \citep{deisenroth_pilco,gal2016improving}, hallucinating training examples \citep{talvitie_hallucination_14,kamyar_surprising,venkatraman_multi_step,oh2015action,bengio2015scheduled}, using ensembles \citep{ensemble_kurutach,model_based_active_exploration}, or model regularization \citep{nan_planning,asadi_lipschitz}. These methods are independent of the multi-step model idea; future work will explore the benefits of combining these advances with \MCube.

\section{Empirical Results}

The goal of our experiments is to investigate if the multi-step model can perform better than the one-step model in several model-based scenarios. We specifically set up experiments in background planning and decision-time planning to test this hypothesis. Note also that we provide code for all experiments in our supplementary material.

\subsection{Background Planning}
As mentioned before, one use case for models is to enable fast value-function estimation. We compare the utility of the one-step model and the multi-step model for this purpose. For this experiment, we used the all-action variant of actor-critic algorithm, in which the value function $\widehat{Q}$ (or the critic) is used to estimate the policy gradient  $\expected{s}{\sum_{a}\nabla_{w}\pi(a|s;w)\widehat{Q}(s,a;\theta)}$ \citep{sutton_policy_gradient,allen_mac}.
Note that it is standard to learn the value function model-free:
$$\theta\leftarrow\theta +\alpha \big(G_1-\widehat Q(s_t,a_t;\theta)\big)\nabla_{\theta}\widehat Q(s_t,a_t,\theta)\ ,$$
where $G_1:=r_t+ \widehat Q(s_{t+1},a_{t+1},\theta)$. \cite{mnih_asynchronous} generalize this objective to a multi-step target $G_H~:=(\sum_{i=0}^{H-1}r_{t+i}) +\widehat Q(s_{t+h},a_{t+h};\theta)$.
Crucially, in the model-free case, we compute $G_H$ using the single trajectory observed during environmental interaction. However, because the policy is stochastic, $G_H$ is a random variable with some variance. To reduce variance, we can use a learned model to generate an arbitrary number of rollouts (5 in our experiments), compute $G_H$ for each rollout, and average them. We compared the effectiveness of both one-step and the multi-step models in generating useful rollouts for learning. To ensure meaningful comparison, for each algorithm, we perform the same number of value-function updates. We used three standard RL domains from Open AI Gym \citep{brockman2016openai}, namely Cart Pole, Acrobot, and Lunar Lander. Results are summarized in Figures \ref{fig:control1} and \ref{fig:control2}.
\begin{figure*}[ht]
  \centering
  \begin{subfigure}[b]{0.3\textwidth}
  \includegraphics[scale=0.4]{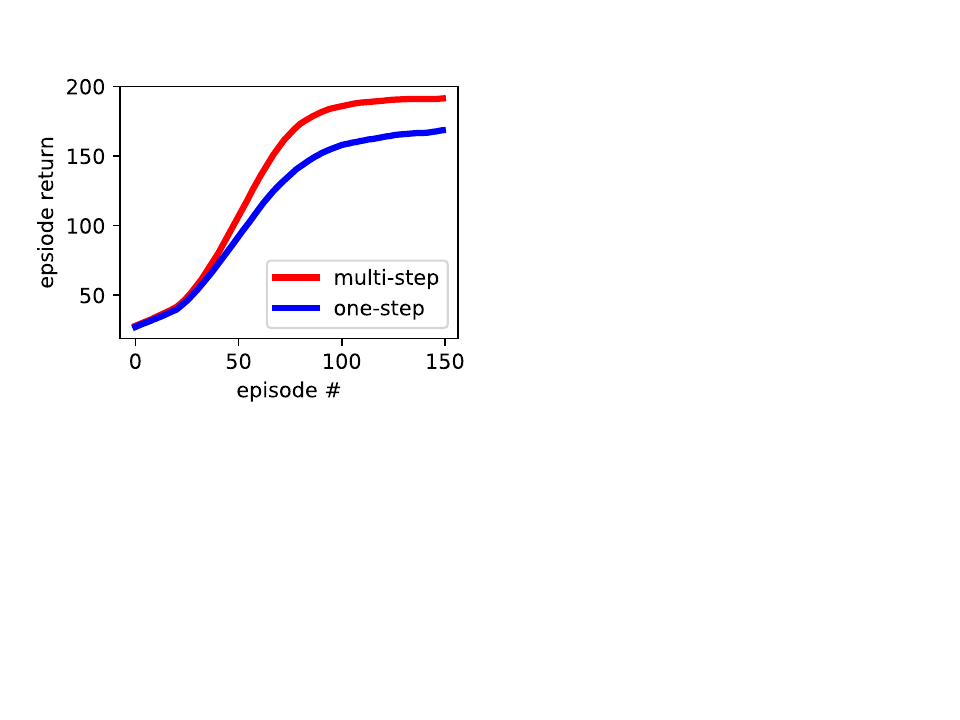}
  \end{subfigure}
    \begin{subfigure}[b]{0.3\textwidth}
  \includegraphics[scale=0.4]{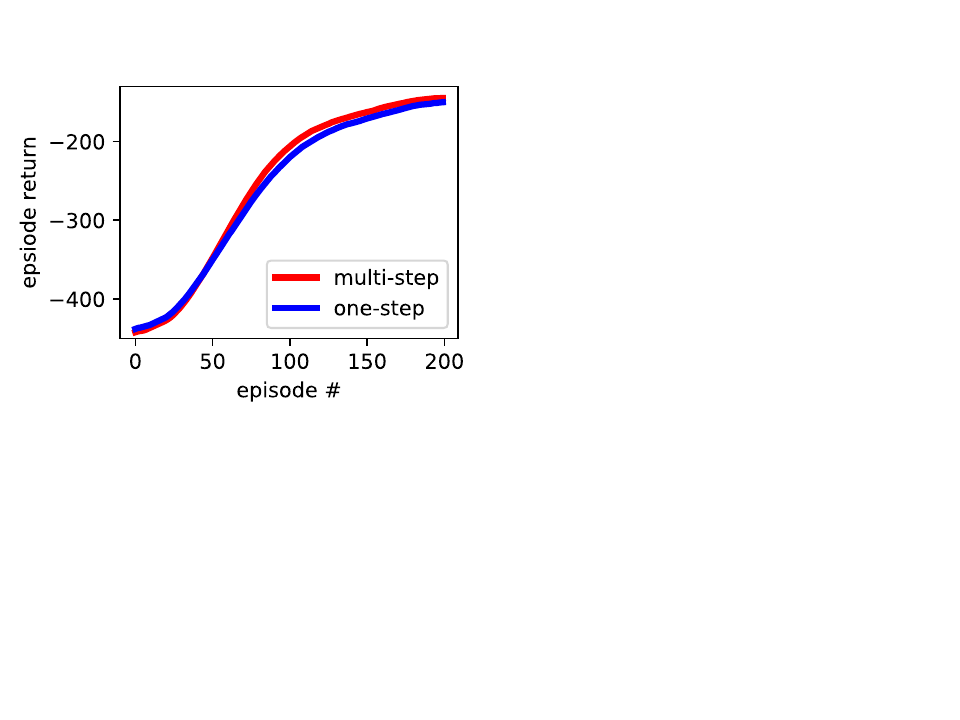}
  \end{subfigure}
    \begin{subfigure}[b]{0.3\textwidth}
  \includegraphics[scale=0.4]{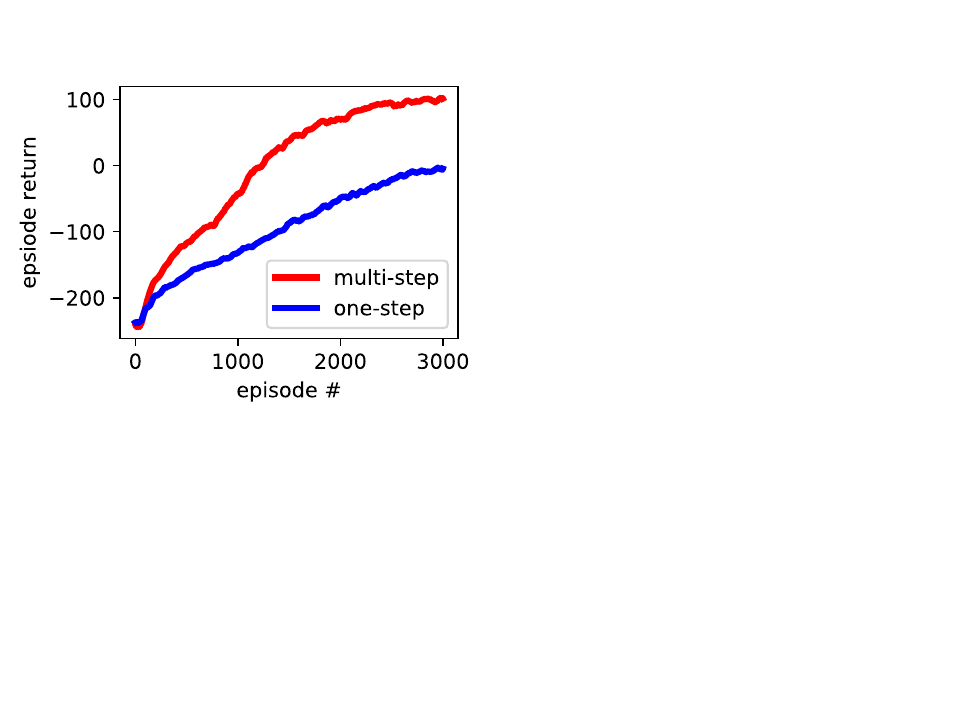}
  \end{subfigure}
  \caption{A comparison of actor critic equipped with the learned models (Cart Pole, Acrobot, and Lunar Lander). We set the maximum look-ahead horizon $H=8$. Results are averaged over 100 runs, and higher is better.  The multi-step model consistently matches or exceeds the one-step model.}
 \label{fig:control1}
\end{figure*}
\begin{figure*}[ht]
  \centering
  \begin{subfigure}[b]{0.3\textwidth}
  \includegraphics[scale=0.4]{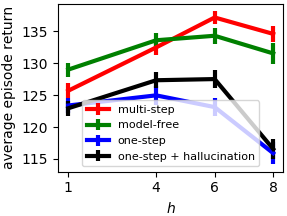}
  \end{subfigure}
    \begin{subfigure}[b]{0.3\textwidth}
  \includegraphics[scale=0.4]{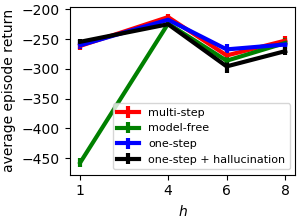}
  \end{subfigure}
    \begin{subfigure}[b]{0.3\textwidth}
  \includegraphics[scale=0.4]{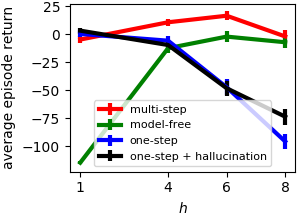}
  \end{subfigure}
  \caption{Area under the curve, which corresponds to average episode return, as a function of the look-ahead horizon $h$. Results for all three domains (Cart Pole, Acrobot, and Lunar Lander)
  are averaged over 100 runs. We add two additional baselines, namely the model-free critic, and a model-based critic trained with hallucination~\cite{talvitie_hallucination_14,venkatraman_multi_step}}.
  \label{fig:control2}
\end{figure*}

To better understand the advantage of the multi-step model, we show per-episode transition error under the two models. Figure~8
(see Appendix) clearly indicates that the multi-step model is more accurate for longer horizons. This is consistent with the theoretical results presented earlier. Note that in this experiment we did not use the model for action selection, and simply queried the policy network and sampled an action from the distribution provided by the network given a state input.
\subsection{Decision-time Planning}

\begin{figure*}
  \centering
  \begin{subfigure}{.45\textwidth}
  \fbox{\includegraphics[width=.75\linewidth]{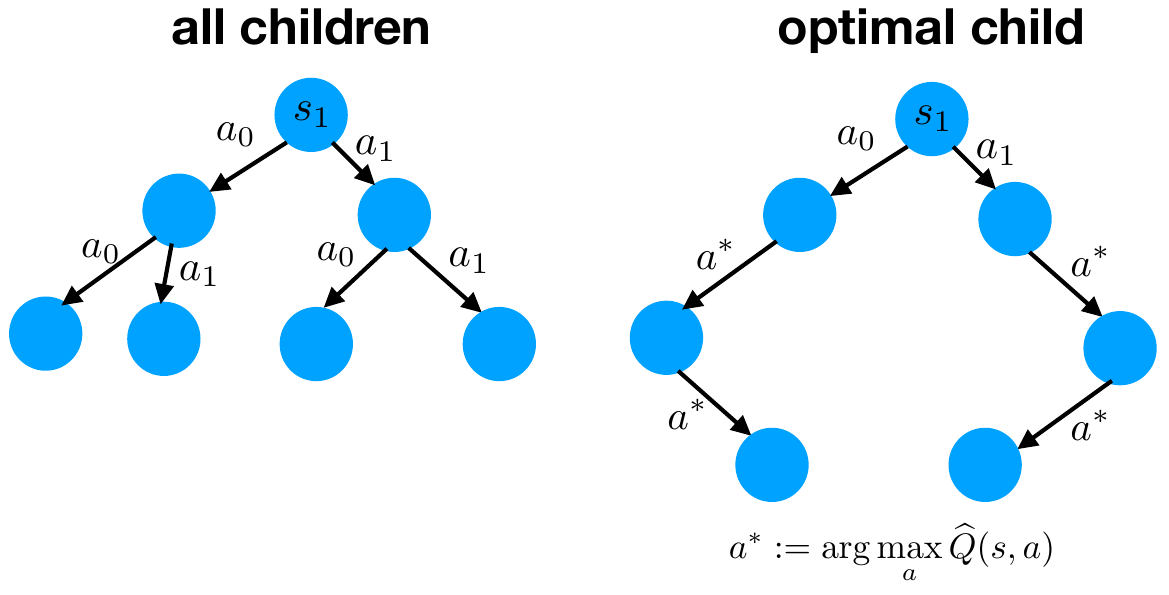}}
  \end{subfigure}%
    \begin{subfigure}{.45\textwidth}
  \fbox{\includegraphics[width=1\linewidth]{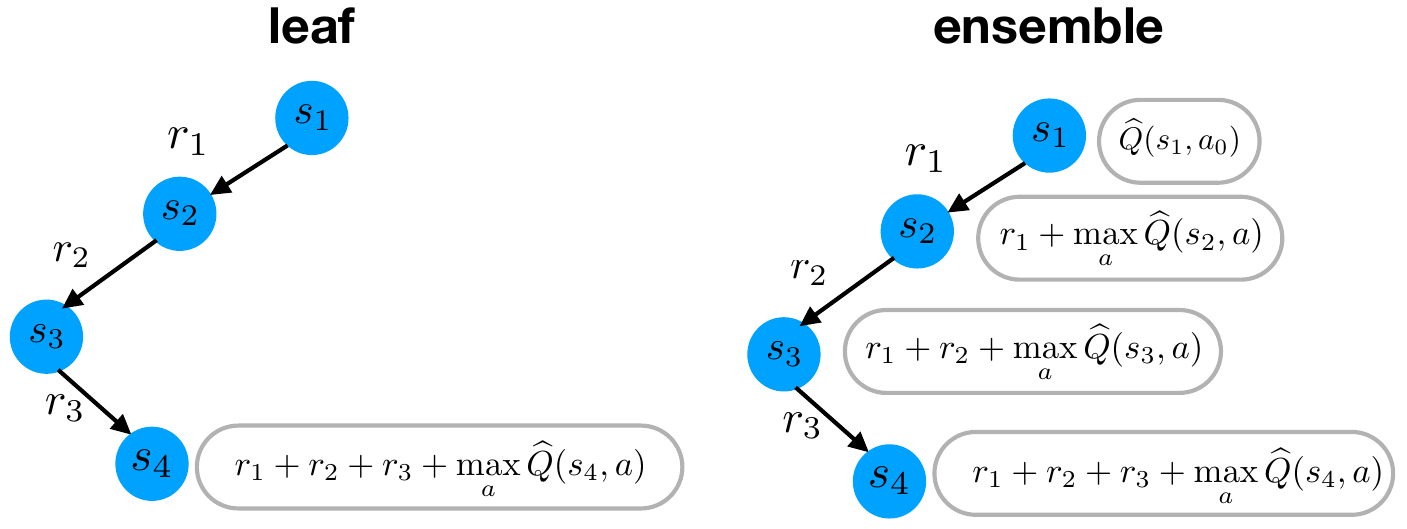}}
  \end{subfigure}
  \caption{Tree construction (left) and action-value estimation (right) strategies. }
 \label{tree_construction}
\end{figure*}
We now use the model for action selection. A common action-selection strategy is to choose  $\arg \max_{a} \widehat{Q}(s,a)$, called the model-free strategy, hereafter. Our goal is to compare the utility of model-free strategy with its model-based counterparts. Our desire is to compare the effectiveness of the one-step model with the multi-step model in this scenario.

A key choice in decision-time planning is the strategy used to construct the tree. One approach is to expand the tree for each action in each observed state \citep{kamyar_surprising}. The main problem with this strategy is that the number of nodes grow exponentially. Alternatively, using a learned action-value function $\widehat Q$, at each state $s$ we can only expand the most promising action $a^*:=\arg \max_a \widehat Q(s,a)$. Clearly, given the same amount of computation, the second strategy can benefit from performing deeper look aheads. The two strategies are illustrated in Figure~\ref{tree_construction} (left).

Note that because the model is trained from experience, it is still only accurate up to a certain depth. Therefore, when we reach a specified planning horizon, $H$, we simply use $\max_a\widehat{Q}(s_H,a)$ as an estimate of future sum of rewards from the leaf node $s_H$. While this estimate can be erroneous, we observed that it is necessary to consider, because otherwise the agent will be myopic in the sense that it only looks at the short-term effects of its actions.

The second question is how to determine the best action given the built tree. One possibility is to add all rewards to the value of the leaf node, and go with the action that maximizes this number. As shown in Figure~\ref{tree_construction} (right), another idea is to use an ensemble where the final value of the action is computed using the mean of the $H$ different estimates along the rollout. This idea is based on the notion that in machine learning averaging many estimates can often lead to a better estimate than the individual ones \citep{schapire2003boosting,caruana_ensemble}.

The two tree expansion strategies, and the two action-value estimation strategies together constitute four possible combinations. To find the most effective combination, we first performed an experiment in the Lunar Lander setting where, given different pretrained $\widehat Q$ functions, we computed the improvement that the model-based policy offers relative to the model-free policy. We trained these $\widehat Q$ function using the DQN algorithm~\citep{DQN} and stored weights every 100 episodes, giving us 20 snapshots of $\widehat{Q}$. The models were also trained using the same amount of data that a particular $\widehat{Q}$ was trained on. We then tested the four strategies (no learning was performed during testing). For each episode, we took the frozen $\widehat Q$ network of that episode, and compared the performance of different policies given $\widehat Q$ and the trained models. In this case, by performance we mean average episode return over 20 episodes.

Results, averaged over 200 runs, are presented in Figure~\ref{policy_test} (left), and show an advantage for the ensemble and optimal-action combination (labeled optimal ensemble). Note that, in all four cases, the model used for tree search was the one-step model, and so this served as an experiment to find the best combination under this model. We then performed the same experiment with the multi-step model as shown in Figure \ref{policy_test} (right) using the best combination (i.e. optimal action expansion with ensemble value computation). We clearly see that \MCube\ is more useful in this scenario as well.

We further investigated whether the superiority in terms of action selection can actually accelerate DQN training as well. In this scenario, we ran DQN under different policies, namely model-free, model-based with the one-step model, and model-based with \MCube. In all cases, we chose a random action with probability $\epsilon=0.01$ for exploration. See Figure \ref{lunar_MBRL}, which again shows the benefit of the multi-step model for decision-time planning in model-based RL.
\begin{figure*}
  \centering
  \begin{subfigure}{.525\textwidth}
\includegraphics[width=.95\linewidth,height=125pt]{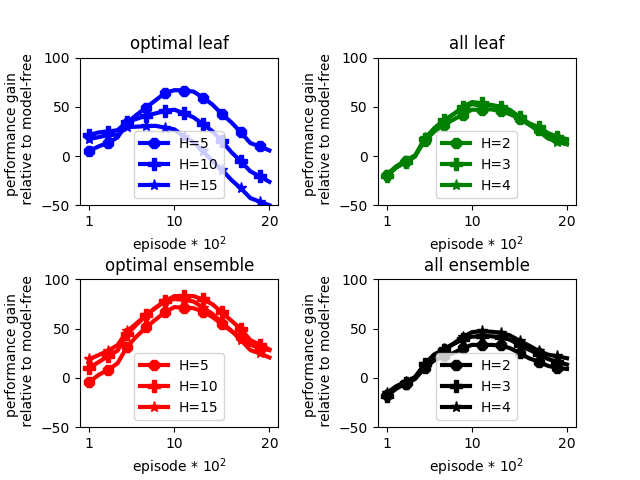}
  \end{subfigure}%
    \begin{subfigure}{.4\textwidth}
  \includegraphics[width=.95\linewidth]{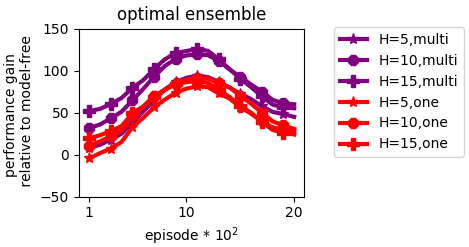}
  \end{subfigure}
\caption{A comparison between tree expansion and value-estimation strategies when using the one-step model for action selection (left). Comparison between the one-step model and \MCube\ for action selection (right). x-axis denotes the $\widehat{Q}$ of agent at that episode, and y-axis denotes performance gain over model-free. Performance is defined as episode return averaged over 20 episodes. Note the inverted-U. Initially, $\widehat{Q}$ and the model are both bad, so model provides little benefit. Towards the end $\widehat{Q}$ gets better, so using the model is not beneficial. However, we get a clear benefit in the intermediate episodes because the model is faster to learn than $\widehat{Q}$.}
\label{policy_test}
\end{figure*}
\begin{figure}
  \begin{minipage}[c]{0.5\textwidth}
    \caption{
      A comparison between the two models in the context of model-based RL. Action selection with the multi-step model can significantly boost sample efficiency of DQN. All models are trained online from the agent's experience. Results are averaged over 100 runs, and shaded regions denote standard errors.
    } \label{lunar_MBRL}
  \end{minipage}\hfill
  \begin{minipage}[c]{0.45\textwidth}
    \includegraphics[width=\textwidth]{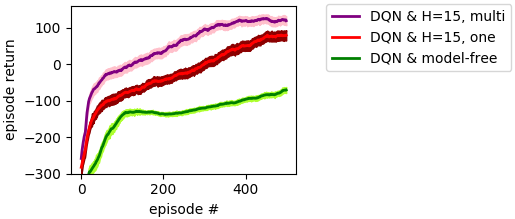}
  \end{minipage}
\end{figure}
\section{Conclusion}

We introduced an approach to multi-step model-based RL and provided results that suggest its promise.
We introduced the \MCube\ model along with a new rollout procedure that exploits it. The combination was proven useful from a theoretical and practical point of view. Together, our results made a strong case for multi-step model-based RL.
\section{Future Work}

An important avenue for future work is to better explore methods appropriate to stochastic environments. (See Appendix for an extension in this direction.) Another thread to explore is an ensemble method enabled by the \MCube\ model. Specifically, we can show that an $h$-step prediction can be estimated in exponentially many ways, and combining these estimates can lead to a better final estimate (see Appendix).

Finally, large pixel-based domains have proven challenging for model-based methods~\citep{machado2018revisiting,kamyar_surprising}, so a future direction is to investigate the effectiveness of the multi-step model in such domains. 

\newpage
\bibliographystyle{named}
\bibliography{refs}
\newpage
\section{Appendix}
\subsection{Proofs}
\begin{figure}
	\centering
	\includegraphics[width=0.5\linewidth]{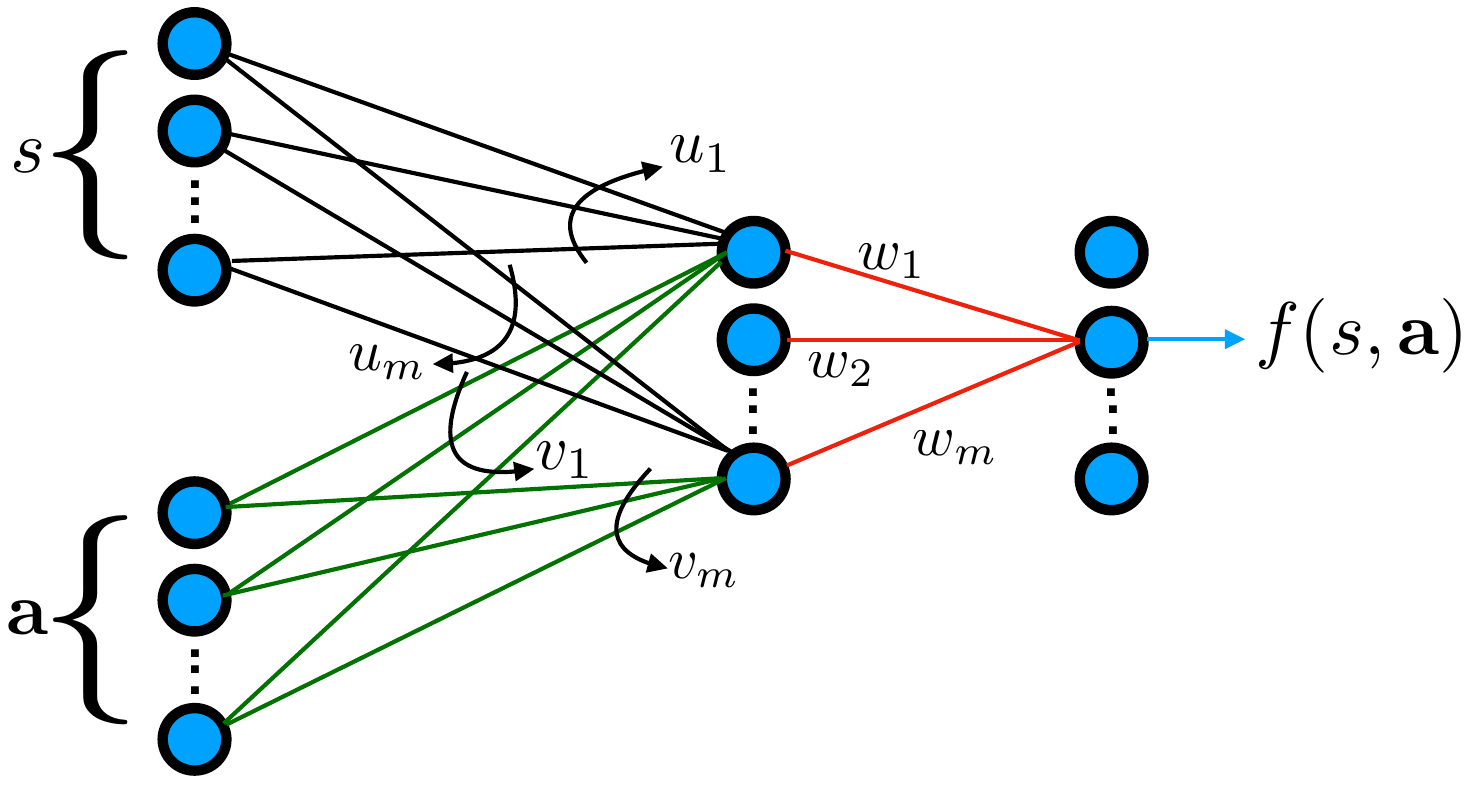}
	\caption{Architecture used to represent a mapping from a state to a single dimension of the next state $h$ steps into the future. We assume $\norm{w}_{1}\leq 1, \norm{u_i}\leq 1$, and $\norm{v_i}\leq 1\ \forall i\in\integer: i\in \{1,m\}$. We show a bound on the Rademacher complexity of this network.}
	\label{fig:nn_architecure}
\end{figure}
\rademacherNeuralNets*
\begin{proof}
First note that:
	\begin{equation*}
	\begin{aligned}
	&\rademacher{\fancyT_h}:=\expected{s^i,\aVec^i,\sigma_{ij}}{\sup_{\That_h \in \fancyT_h}\frac{1}{|D|}\sum_{i=1}^{|D|}\sum_{j=1}^d\sigma_{ij}\That_h(s^i,\aVec^i_h)_j}\\
	&=\expected{s^i,\aVec^i_h,\sigma_{ij}}{\sup_{\That_h:=\langle f_1,...,f_d \rangle}\frac{1}{|D|}\sum_{i=1}^{|D|}\sum_{j=1}^d\sigma_{ij}f_{j}(s^i,\aVec^i_h)}\\
	&=\expected{s^i,\aVec^i_h,\sigma_{i}}{\sum_{j=1}^d\sup_{f\in \fancyF}\frac{1}{|D|}\sum_{i=1}^{|D|}\sigma_{i}f(s^i,\aVec^i_h)}\\
	&= d\ \expected{s^i,\aVec^i_h,\sigma_{i}}{\sup_{f\in \fancyF}\frac{1}{|D|}\sum_{i=1}^{|D|}\sigma_{i}f(s^i,\aVec^i_h)}=d\ \rademacher{\fancyF}\ .
	\end{aligned}
\end{equation*}
So we rather focus on the set of scalar-valued functions $\fancyF$. Moreover, to represent the model, we set $\fancyF$ to be neural network with bounded weights as characterized in Figure \ref{fig:nn_architecure}.
\begin{equation*}
\begin{aligned}
	&\rademacher{\fancyF}:=\expected{s^i,{\aVec}^i_h,\sigma_i}{  \sup_{f\in\fancyF} \frac{1}{|D|} \sum_{i=1}^{|D|} \sigma_i f(s^i,\aVec^i)}\\
	&= \expected{s^i,{\aVec}^i_h,\sigma_i}{\sup_{w_j,u_j,v_j} \frac{1}{|D|} \sum_{i=1}^{|D|} \sigma_i  \sum_{j=1}^{m}w_j \ReLU(u^\top_j s^i+v^\top_j \aVec^i_h)}\\
	&\leq  \expected{s^i,{\aVec}^i_h,\sigma_i}{ \sup_{w_j,u_j,v_j} \frac{1}{|D|}  \sum_{j=1}^{m}w_j \big|\sum_{i=1}^{|D|} \sigma_i\ReLU(u^\top_j s^i+v^\top_j \aVec^i_h)\big|}\\
	&=  \expected{s^i,{\aVec}^i_h,\sigma_i}{\sup_{u,v} \frac{1}{|D|}   \big|\sum_{i=1}^{|D|} \sigma_i\ReLU(u^\top s^i+v^\top \aVec^i_h)\big| }\\
\end{aligned}
\end{equation*}
Now note that $Lip(ReLU)=1$, so by Theorem 4.12 in \cite{ledoux_probability}:
\begin{equation*}
\begin{aligned}
	&\leq  \expected{s^i,{\aVec}^i_h,\sigma_i}{\sup_{u,v} \frac{1}{|D|}   \big|\sum_{i=1}^{|D|} \sigma_i (u^\top s^i+v^\top \aVec^i_h)\big| }\\
	&=\expected{s^i,\sigma_i}{\sup_{u} \frac{1}{|D|}   \big|u^\top\sum_{i=1}^{|D|} \sigma_i s^i\big|}+\expected{\aVec^i_h,\sigma_i}{\sup_{v} \frac{1}{|D|}   \big|v^\top\sum_{i=1}^{|D|}  \sigma_i \aVec^i_h\big|}\\
	&\text{$\norm{u}_1\leq 1$ and  $\norm{v}_1\leq 1$, so using Cauchy-Shwartz:}\\
	&\leq \expected{s^i,\sigma_i}{\frac{1}{|D|} \norm{\sum_{i=1}^{|D|} \sigma_i s_i}_2 }+\expected{\aVec^i_h,\sigma_i}{ \frac{1}{|D|} \norm{\sum_{i=1}^{|D|}  \sigma_i \aVec_i}_2 }\\
	&\text{Due to Jensen's inequality for concave function $f(x)=\sqrt{x}$:}\\
	&\leq \expected{s^i}{\frac{1}{|D|} \sqrt{\expected{\sigma_i}{\norm{\sum_{i=1}^{|D|} \sigma_i s^i}_{2}}} }+\expected{\aVec^i_h}{ \frac{1}{|D|} \sqrt{\expected{\sigma_i}{\norm{\sum_{i=1}^{|D|} \sigma_i \aVec^i_h}_{2}}} }\\
	&=\frac{1}{|D|}\Big(\expected{s^i}{\sqrt{\sum_{i=1}^{|D|}\norm{ s^i}_{2}} }+\expected{\aVec^i_h}{\sqrt{\sum_{i=1}^{|D|}\norm{ \aVec^i_h}_{2}} }\Big)
	\end{aligned}
\end{equation*}
\end{proof}

\begin{lemma}{\cite{asadi_equivalence}}
	Define the $H$-step value function: $$V^{\pi}_{H}(s):=\expected{s_i,a_i}{\sum_{i=1}^{H} R(s_i,a_i)}\ ,$$
	and assume a Lipschitz reward function $R$ with constant $Lip^{\fancyA}(R)$. Then:
	$$\Lip{}{}{(V^{\pi}_{H})}\leq\Lip{}{\fancyA}{(R)} H\ .$$
	\label{lemma:lipschitz_value}
\end{lemma}
\oneStepValueError*
\begin{proof}
\begin{equation*}
	\begin{aligned}
		& \Big|\expected{s_1}{V^{\pi}_{H}(s_1)-\widehat V^{\pi}_{H}(s_1)}\Big|\\
		&=\Big|\expected{s_1,a_1}{V^{\pi}_{H-1}\big(T_1(s_1,a_1)\big)-\widehat V^{\pi}_{H-1}\big(\widehat T_1(s_1,a_1)\big)}\Big|\\
		&=\Big|\expected{s_1,a_1}{V^{\pi}_{H-1}\big(T_1(s_1,a_1)\big)- \widehat V^{\pi}_{H-1}\big(T_1(s_1,a_1)\big)\\
		&\quad + \widehat V^{\pi}_{H-1}\big(T_1(s_1,a_1)\big) +\widehat V^{\pi}_{H-1}\big(\widehat T_1(s_1,a_1)\big)}\Big|\\
		&\leq \Big|\expected{s_1,a_1}{V^{\pi}_{H-1}\big(T_1(s_1,a_1)\big)- \widehat V^{\pi}_{H-1}\big(T_1(s_1,a_1)\big)}\Big|\\
		&\quad +\Big|\expected{s_1,a_1}{ \widehat V^{\pi}_{H-1}\big(T_1(s_1,a_1)\big) -\widehat V^{\pi}_{H-1}\big(\widehat T_1(s_1,a_1)\big)}\Big|\\
		&\leq \Big|\expected{s_2}{V^{\pi}_{H-1}(s_2)- \widehat V^{\pi}_{H-1}(s_2)}\Big|\\
		&\quad +\underbrace{\expected{s_1,a_1}{ \big|\widehat V^{\pi}_{H-1}\big(T_1(s_1,a_1)\big) -\widehat V^{\pi}_{H-1}\big(\widehat T_1(s_1,a_1)\big)}\big|}_{\leq\Lip{}{\fancyA}{(R)}(H-1)\expected{s_1,a_1}{ \norm{T_1(s_1,a_1) -\widehat T_1(s_1,a_1)}} \  \text{(due to Lemma \ref{lemma:lipschitz_value}})}\\
	\end{aligned}
\end{equation*}	
We reach the desired result by expanding the first term for $H-2$ more times.
\end{proof}

\multiStepValueError*
\begin{proof}
	\begin{equation*}
		\begin{aligned}
			& \Big|\expected{s_1}{V^{\pi}_{H}(s_1)-\widehat V^{\pi}_{H}(s_1)}\Big|\\
			&=\Big|\expected{s_1,\aVec_h,\aVec'_h}{ \sum_{h=1}^{H} R\big(T_h(s_1,\aVec_h)\big)- \sum_{h=1}^H R\big(\widehat T_h(s_1,\aVec'_h)\big)} \Big|\\
			& \leq \sum_{h=1}^{H}\expected{s_1,\aVec_h,\aVec'_h}{ \big|R\big(T_h(s_1,\aVec_h)\big) -R\big(\widehat T_h(s_1,\aVec'_h)\big)\big|}\\
			& \textup{assuming}\ \pi\big(\cdot|T(s,\aVec_h)\big)=\pi\big(\cdot|\That(s,\aVec'_h)\big):\\
			& =\sum_{h=1}^{H}\expected{s_1,\aVec_h}{ \big|R\big(T_h(s_1,\aVec_h)\big) -R\big(\widehat T_h(s_1,\aVec   `_h)\big)\big|}\\
			& \textup{due to Lemma}\ \ref{lemma:lipschitz_value}:\\
			& \leq \Lip{}{\fancyA}{(R)}\sum_{h=1}^{H}\expected{s_1,\aVec_h}{\norm{T_h(s_1,\aVec_h) -\widehat T_h(s_1,\aVec_h)}}\ .
		\end{aligned}
	\end{equation*}
\end{proof}

\RademacherLemma*
\begin{proof}
	We heavily make use of techniques provided by \namecite{bartlett_rademacher}. First note that, due to the definition of $\sup$, we clearly have:
	\begin{equation}
			\begin{aligned}
		&\expected{s,\aVec}{\norm{T_h(s,\aVec)-\That_h(s,\aVec)}_{1}} \\
		&\leq \frac{1}{|D|}\sum_{i=1}^{|D|}\norm{T_h(s^i,\aVec^i_h)-\That_h(s^i,\aVec^i_h)}_{1}\\
		&+\sup_{\That_h \in \fancyT}\Big\{ \expected{s,\aVec_h}{\norm{T_h(s,\aVec_h)-\That_h(s,\aVec_h)}_{1}}  \\ &\quad -\frac{1}{|D|}\sum_{i=1}^{|D|}\norm{T_h(s^i,\aVec^i_h)-\That_h(s^i,\aVec^i_h)}_{1}\Big\}\ .
	\end{aligned}
	\label{eq:sup_property}
	\end{equation}
We define $\Phi$ to be the $\sup$ in the right hand side of the above bound:
\begin{equation*}
\begin{aligned}
    & \Phi(s^{1...|D|},\aVec_h^{1...|D|}):=\sup_{\That_h \in \fancyT}\Big\{ \expected{s,\aVec}{\norm{T_h(s,\aVec_h)-\That_h(s,\aVec_h)}_{1}}  \\ &\quad -\frac{1}{|D|}\sum_{i=1}^{|D|}\norm{T_h(s^i,\aVec^i_h)-\That_h(s^i,\aVec^i_h)}_{1}\Big\}\ .\\
\end{aligned}
\end{equation*}

We can bound $\expected{s^i,\aVec^i_h}{\Phi}$ in terms of $\rademacher{\fancyT_h}$:
\begin{equation*}
\begin{aligned}
	&\expected{s^i,\aVec^i_h}{\Phi}=\expected{s^i,\aVec^i_h}{\sup_{\That_h \in \fancyT} \big \{\expected{s,\aVec_h}{\norm{T_h(s,\aVec_h)-\That_h(s,\aVec_h)}_{1}}\nonumber\\
	&-  \frac{1}{|D|}\sum_{i=1}^{|D|}\norm{T_h(s^i,\aVec^i_h)-\That_h(s^i,\aVec^i_h)}_{1}\big\}}\nonumber\\
	&=\expected{s^i,\aVec^i_h}{\sup_{\That_h \in \fancyT}\big\{\expected{s'^i,\aVec'^i_h}{\sum_{i=1}^{|D|} \frac{1}{|D|}\big(\norm{T_h(s'^i,\aVec'^i_h)-\That_h(s'^i,\aVec'^i_h)}_{1}\nonumber\\
	&-\norm{T_h(s^i,\aVec^i_h)-\That_h(s^i,\aVec^i_h)}_{1}\Big]\big)\big\}}}\nonumber\\
	& \text{Due to Jensen's inequality:}\nonumber\\
	&\leq \expected{s^i,\aVec^i,s'^i,\aVec'^i}{\sup_{\That_h \in \fancyT}\big\{\frac{1}{|D|}\sum_{i=1}^{|D|} \big(\norm{T_h(s'^i,\aVec'^i_h)-\That_h(s'^i,\aVec'^i_h)}_{1}\nonumber\\
	&-\norm{T_h(s^i,\aVec^i_h)-\That_h(s^i,\aVec^i_h)}_{1}\big)\big\}}\nonumber\\
	& \text{Due to $\sigma_i$ uniformly randomly chosen from \{-1,1\}}:\nonumber\\
	&= \expected{s^i,\aVec^i_h,{s'^i},\aVec'^i_h,\sigma_i}{\!\sup_{\That_h \in \fancyT}\big\{\frac{1}{|D|}\sum_{i=1}^{|D|} \!\sigma_i\!\big(\!\norm{T_h(s'_i,\aVec_i')-\That_h(s'^i,\aVec'^i)}_{1}\nonumber\\
	&-\norm{T_h(s^i,\aVec^i_h)-\That_h(s^i,\aVec^i_h)}_{1}\big)\big\}}\nonumber\\
	&\leq 2\ \expected{s^i,\aVec^i_h,\sigma_i}{\sup_{\That_h\in \fancyT}\frac{1}{|D|}\sum_{i=1}^{|D|}\sigma_i\norm{T_h(s^i,\aVec^i)-\That_h(s^i,\aVec^i_h)}_{1}}\nonumber\\
	& \text{Due to Corollary 4 of \namecite{maurer_vector}:} \nonumber\\
	&\leq 2\sqrt{2}\ \expected{s^i,\aVec^i_h,\sigma_{ij}}{\sup_{\That_h \in \fancyT}\frac{1}{|D|}\sum_{i=1}^{|D|}\sum_{j=1}^d\sigma_{ij}\That_h(s^i,\aVec^i_h)_j}\nonumber\\
	&= 2\sqrt{2}\rademacher{\fancyT_h} \ . 
	\label{eq:bound_phi_with_rademacher}
	\end{aligned}
\end{equation*}
Next note that the function $\Phi$ satisfies:
$$|\Phi(s^{1...i...|D|},\aVec^{1...i...|D|}_h)-\Phi(s^{1...i'...|D|},\aVec^{1..i'...|D|}_h)|\leq\frac{2d}{|D|}\ .$$
So using MacDiarmid's inequality:
\begin{equation*}
	\text{Pr}(\Phi\leq\expected{}{\Phi}+d\sqrt{\frac{\ln{\frac{1}{\delta}}}{|D|}})\geq 1-\delta
	\label{eq:MacDiarmid}
\end{equation*}
Combining the two previous results, we can write:
\begin{equation}
	\text{Pr}(\Phi\leq 2\sqrt{2}\rademacher{\fancyT}  +d\sqrt{\frac{\ln{\frac{1}{\delta}}}{|D|}})\geq 1-\delta
	\label{eq:bound_phi_with_rad}
\end{equation}
Finally, using (\ref{eq:sup_property}) and (\ref{eq:bound_phi_with_rad}) we can conclude the proof.\end{proof}
\finalTheorem*
\begin{proof}
Starting from the one-step case, from Theorem \ref{theorem:value_error_with_one_step_model} we have:
\begin{eqnarray*}
\Big|\expected{s_1}{V^{\pi}_{H}(s_1)-\widehat V^{\pi}_{H}(s_1)}\Big|
			&\leq&  \Lip{}{\fancyA}{(R)}\sum_{h=1}^{H-1}(H-h)\expected{s_h,a_h}{ \norm{T_1(s_h,a_h) -\widehat T_1(s_h,a_h)}}\\
			&& (\textrm{from Lemma \ref{lemma:rademacher_neural_nets}})\\
			&\leq&\Lip{}{\fancyA}{(R)}\sum_{h=1}^{H-1}(H-h)( \frac{\!2\sqrt{2}d}{|D|}\!\big(\expected{s^i}{  \!\sqrt{\!\sum_{i=1}^{|D|}\norm{ s^i}_{2}} }\!+\!1\big)
	+d\ln{\sqrt{\frac{1/\delta}{|D|}}}+\Delta\ )\\
	&=&\Lip{}{\fancyA}{(R)}\sum_{h=1}^{H-1}(H-h)( \frac{\!2\sqrt{2}d}{|D|}\!\expected{s^i}{  \!\sqrt{\!\sum_{i=1}^{|D|}\norm{ s^i}_{2}} }
	+d\ln{\sqrt{\frac{1/\delta}{|D|}}}+\Delta\ )\\
	&&+\ \Lip{}{\fancyA}{(R)}\sum_{h=1}^{H-1}(H-h)\frac{\!2\sqrt{2}d}{|D|}\\
	&=&\frac{H(H-1)}{2}C_1+\frac{H(H-1)}{2}C_2
\end{eqnarray*}
Moving on to the multi-step case, from Theorem \ref{theorem:value_error_with_multi_step_model} we get:
\begin{eqnarray*}
\Big|\expected{s_1}{V^{\pi}_{H}(s_1)-\widehat V^{\pi}_{H}(s_1)}\Big|&\leq&  \Lip{}{\fancyA}{(R)}\sum_{h=1}^{H-1}\expected{s_1,\aVec_{h}}{ \norm{T_h(s_1,\aVec_h) -\That_h(s_1,\aVec_{h})}}\\
&& (\textrm{from Lemma \ref{lemma:rademacher_neural_nets}})\\
&\leq&\Lip{}{\fancyA}{(R)}\sum_{h=1}^{H-1}( \frac{\!2\sqrt{2}d}{|D|}\!\big(\expected{s^i}{  \!\sqrt{\!\sum_{i=1}^{|D|}\norm{ s^i}_{2}} }\!+\!h\big)
	+d\ln{\sqrt{\frac{1/\delta}{|D|}}}+\Delta\ )\\
&=&\Lip{}{\fancyA}{(R)}\sum_{h=1}^{H-1}( \frac{\!2\sqrt{2}d}{|D|}\!\expected{s^i}{  \!\sqrt{\!\sum_{i=1}^{|D|}\norm{ s^i}_{2}} }
	+d\ln{\sqrt{\frac{1/\delta}{|D|}}}+\Delta\ )\\
	&&+\ \Lip{}{\fancyA}{(R)}\sum_{h=1}^{H-1}h\frac{\!2\sqrt{2}d}{|D|}\\
	&=&(H-1)C_1+\frac{H(H-1)}{2}C_2
\end{eqnarray*}
\end{proof}

\begin{figure}[H]
	\centering
	\includegraphics[width=0.6\linewidth]{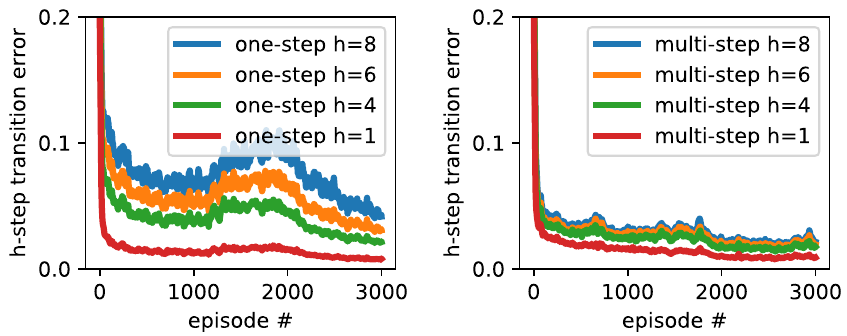}
	\caption{$h$-step model error (down is good) on the Lunar Lander domain under the one-step (left) and the multi-step (right) model. For each new episode, before using the episode for training, we compute $h$-step predictions given all observed states and executed action sequences, and compare the predictions with the states observed $h$ steps later in the episode. The multi-step model outperforms the one-step model.}
\label{fig:lunar_accuracy}
\end{figure}

\subsection{Future Work: An Stochastic Extension}

One limitation of the model so far is that it is deterministic. Here, we introduce a stochastic extension to remove this limitation. Our model is based on the one-step EM model introduced by \namecite{asadi_lipschitz}. In this case, for each value of $h\in [1,H]$, rather than learning a single function $\That_h$, we train $M$ functions each denoted $\That_{hm}$, to capture different \textit{modes} of the $h$-step transition dynamic. The multi-step transition model, $\That$, is then defined as
$$\That:=\big\{\That_{hm}|h\in [1,H],m \in [1,M]\big\}\ ,$$
plus $h$ probability distributions each over $M$ functions $\Pr_h(\That_{h:})$. Each function $\That_{hm}$ is parameterized by a neural network $\That_{hm}(s,\aVec_h;W_{hm})$. For a single $h$, we train these $M$ functions using an Expectation-Maximization (EM) algorithm \cite{dempster_EM}. For more details on the EM algorithm see \namecite{asadi_lipschitz}, but, for completeness, we provide the M-step and the E-step of the algorithm. The M-step at iteration $t$ solves:
$$\underset{W_{h,1:m}}{\textup{maximize}} \sum_{i=1}^{n}\!\sum_{m=1}^M\! q_{t-1}(m|s^i,\aVec_h^i,s'^{i})\!\log \Pr(s^i,\aVec_h^i,s'^{i},m;W_{hm})$$
and also sets:
$$\Pr(\That_{hm})\leftarrow\sum_{i}\Pr(m|s^i,\aVec^i_h,s'^i)\ .$$
The E-step updates the posterior $q$:
$$q_{t}(m|s^i,\aVec^i_h,s'^{i})\leftarrow \frac{\Pr(s^i,\aVec^i_h,s'^{i}|m;W_{hm})}{\sum_{m}\Pr(s^i,\aVec^i_h,s'^{i}|m;W_{hm})}, $$
giving rise to a new optimization problem to solve in the M-step of iteration $t+1$. Finally, our implementation uses:
$$\Pr(s,\aVec_h,s'|m;W_{hm}):=\mathcal{N}\Big(\big|s'-\That_{hm}(s,\aVec_h;W_{hm})\big|,\sigma^2\Big).$$
We apply this algorithm in a variant of the gridworld (mini-Pacman) domain \cite{moerland_multi_modal}. Details of the domain are in supplementary material, but at a high level, the goal is for the agent to start in the bottom-left and get to the top-right corner of the grid, while avoiding a randomly moving ghost. 
\begin{figure}
\centering
\begin{subfigure}[b]{0.17\textwidth}
   \includegraphics[width=1\linewidth]{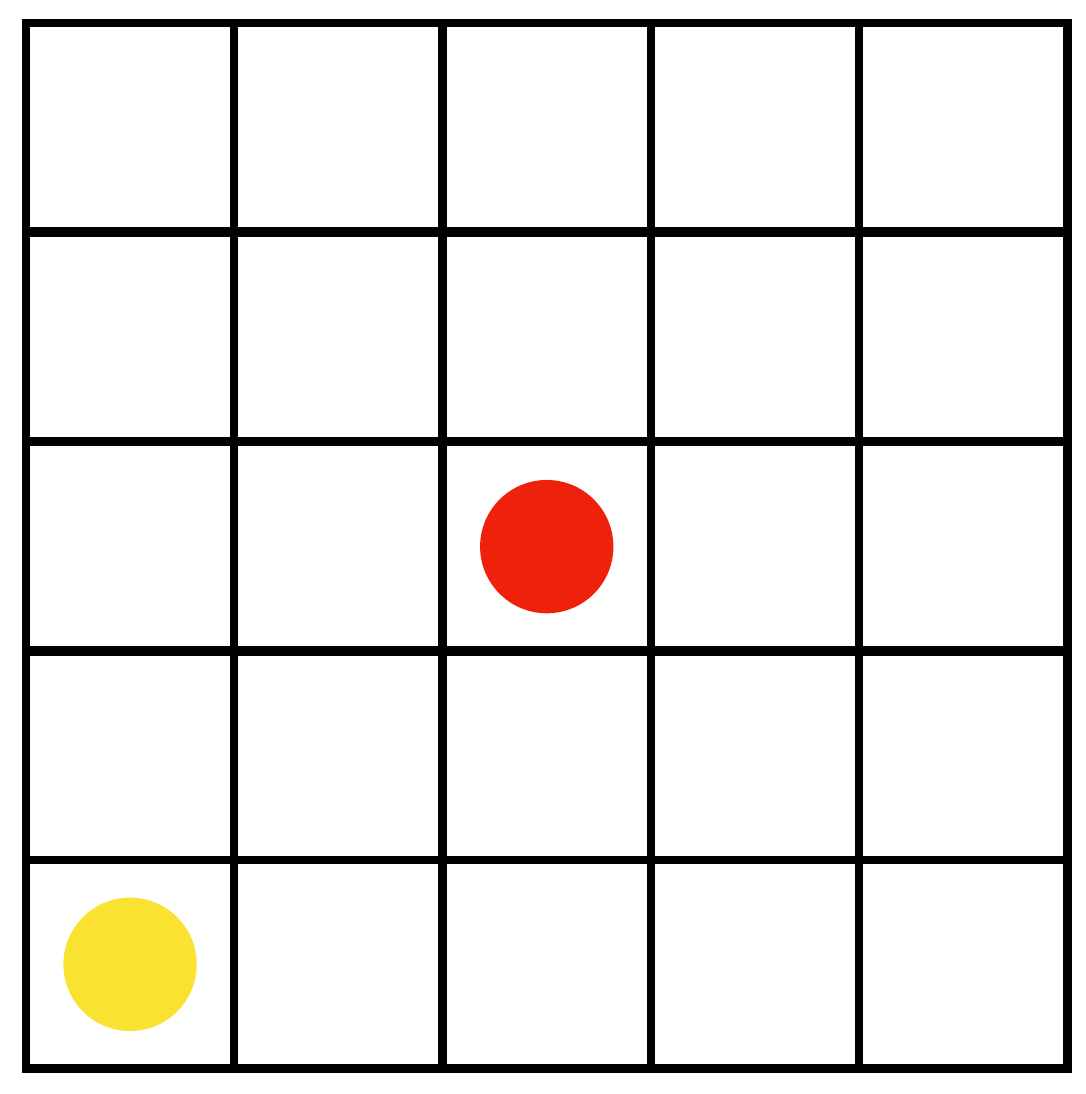}
   \caption{}
\end{subfigure}
\begin{subfigure}[b]{0.17\textwidth}
   \includegraphics[width=1\linewidth]{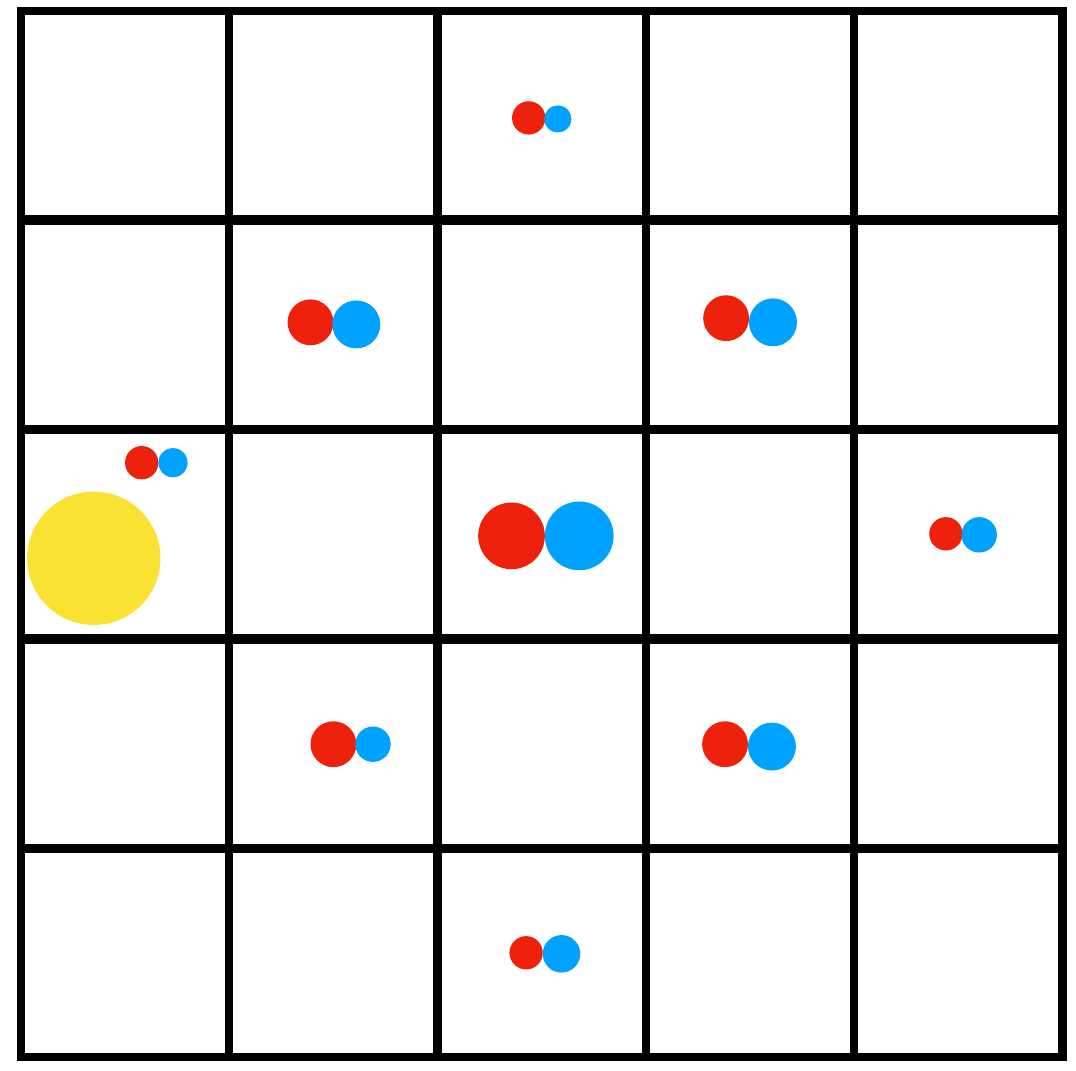}
   \caption{}
\end{subfigure}
\caption[]{A starting state (a), and predicted (blue) and true (red) next state (b) given $\aVec_2=\{\textup{up,up}\}$. Radius of a circle denotes the probability of the circle ending up in the state. The EM model accurately captures the two-step dynamics.}
\label{fig:em_predictions}
\end{figure}
We first show, in Figure \ref{fig:em_predictions}, that the agent can learn a 2-step model of the environment almost perfectly. We then use the model for action selection in Q-learning \cite{rummery_q_learning}, where instead of a greedy policy with respect to the Q-function, we build and search a tree using our learned transition model. In each case, the model was learned using a dataset $\langle s^{i},\aVec_2^{i},s'^{i} \rangle$ from 100 episodes of a random policy interacting in the domain. Having searched the tree up to $H=2$, we then select the action with highest utility. In Figure \ref{fig:em_returns}, we see a clear advantage for the EM model relative to the baselines.
\label{section:stochastic}
\begin{figure}
	\centering
	\includegraphics[width=0.4\linewidth]{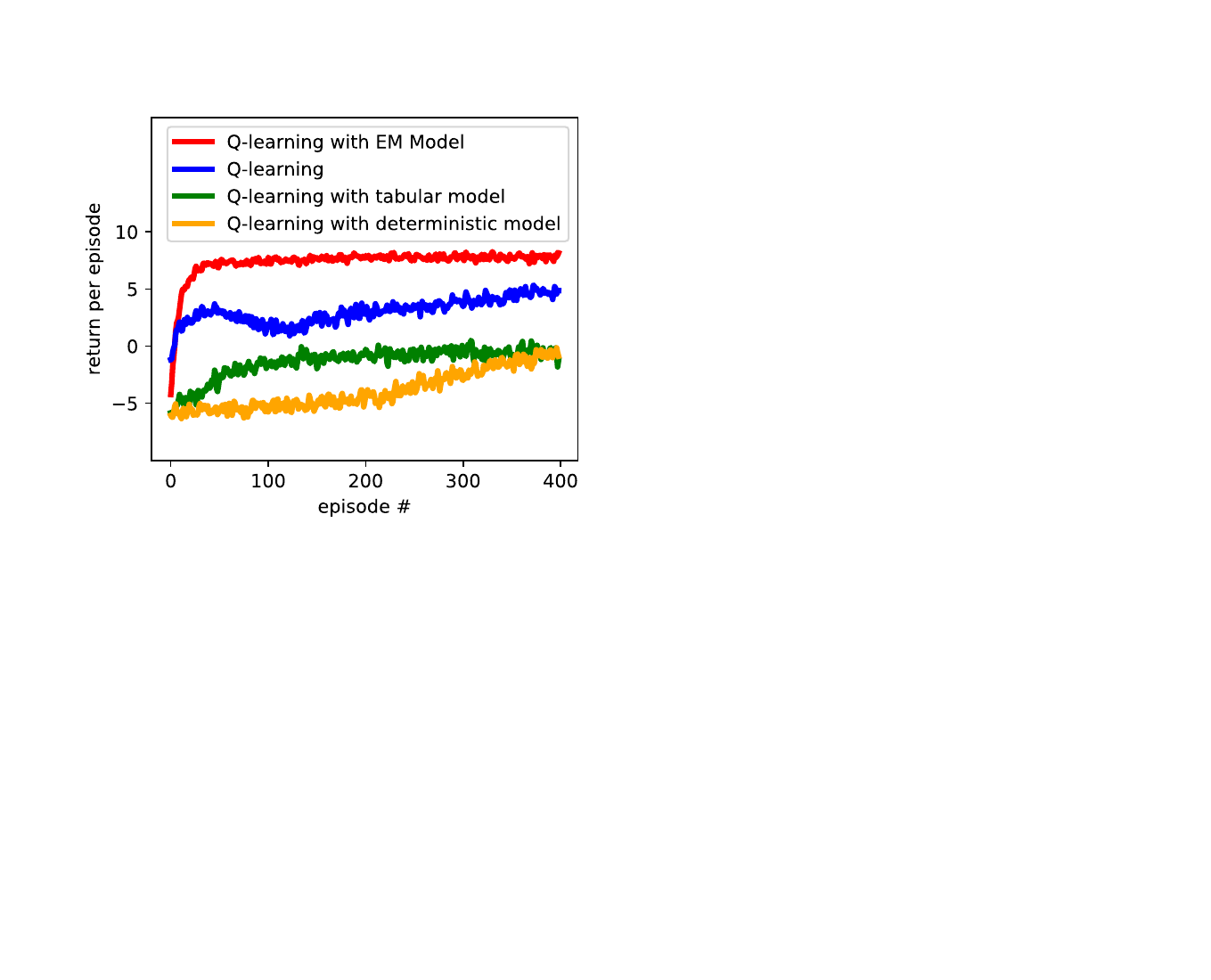}
	\caption{An evaluation of effectiveness of different transition models for tree search on the gridworld. The deterministic model fails to provide useful predictions as the average prediction is for the ghost to not move at all. }
	\label{fig:em_returns}
\end{figure}

\subsection{Future Work: An Ensemble Extension}
A promising idea for future work is a new ensemble technique articulated below. Suppose our goal is to predict the state after following $\aVec_{H}$ in $s_1$. Two obvious paths are:
$$\hat{s}_{H+1}\leftarrow \That_H(s_1,\aVec_H)\ ,\ \quad  \hat{s}_{H+1}\leftarrow (\That_1)^{H}(s_1,\aVec_H)\ .$$
However, there exist more paths, such as using $\That_2$ twice, followed by $\That_{H-4}$:
$$\hat{s}_{H+1}\leftarrow \That_{H-4}\big((\That_2)^2(s_1,\aVec_4),\aVec_{5:H}\big)\ .$$
In fact, the number of paths grows exponentially with $H$: Let $C(H)$ be the number of paths. We have the following distinct paths: First use $\That_h$, then do the remaining $H-h$ steps in any arbitrary path. There are $C(H-h)$ such paths, So:
$$C(H)=1+\sum_{h=1}^{H-1}C(H-h)=1+\sum_{h=1}^{H-1}C(h)=2^{H-1} \ .$$
We can thus use a subset of these $2^{H-1}$ paths and average the results together as is common with other ensemble methods~\cite{caruana_ensemble}. Below, we report a preliminary experiment evaluating this idea.

We ran the actor-critic algorithm in the Acrobot domain. We trained the multi-step model with maximum horizon $H=8$. For each episode, we computed the 8-step transition error as a function of the number of paths used to compute the prediction. Note that we sampled a path uniformly at random from $2^{8-1}$ possible paths. Results are presented in Figure \ref{ensemble_experiment}, and show that predictions get more accurate as we average over more paths. 

Note that prior works on ensemble methods in model-based RL exist \cite{moore_prioritized,ensemble_kurutach}, but only consider combinations of one-step models. Here, we only scratched the surface of this idea, and we leave further exploration for future work.
\begin{figure}
	\centering
	\includegraphics[width=0.4\linewidth]{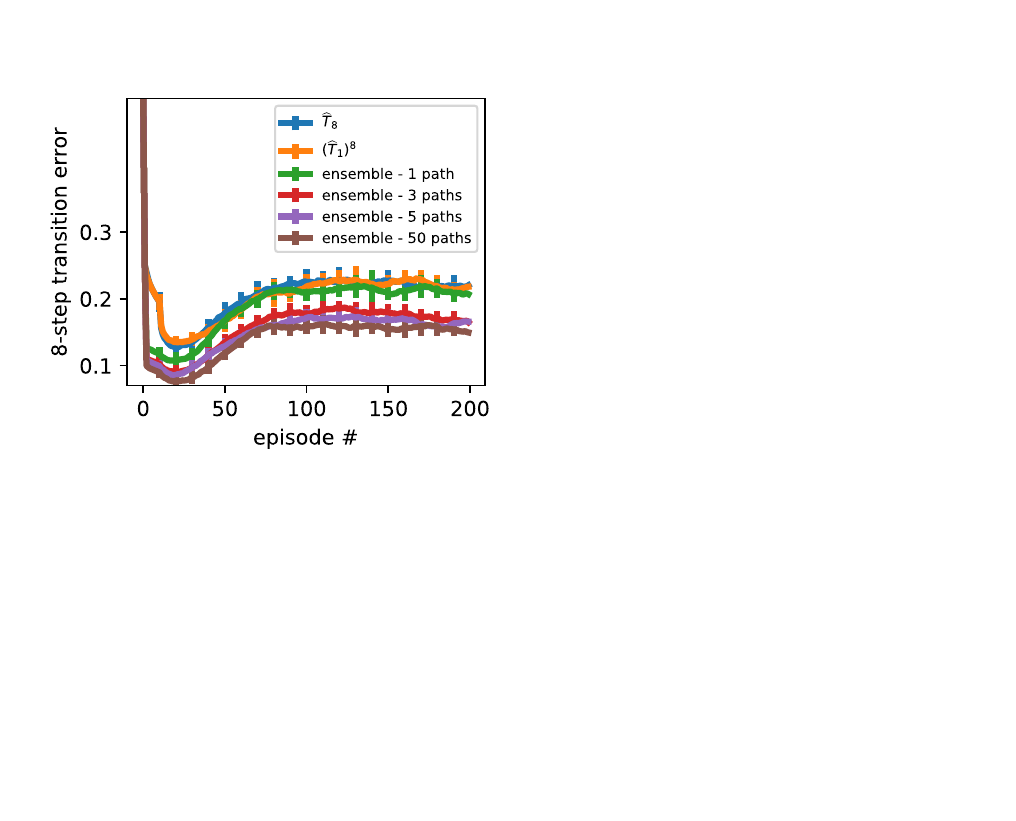}
	\caption{8-step model accuracy. Accuracy of ensembles increases with more sampled paths, and outperforms the accuracy of $\That_8$ as well as $(\That_1)^8$. Note that, due to the nature of this task, as the learner gets better at solving the task, it experiences more diverse states, hence the decrease in accuracy in later episodes.}
\label{ensemble_experiment}
\end{figure}

\end{document}